\documentclass[11pt]{article}
\usepackage{coling2020}
\usepackage{times}
\usepackage{url}
\usepackage{latexsym}

\usepackage{wrapfig}
\usepackage{algorithm}
\usepackage{algpseudocode}
\usepackage{booktabs}
\usepackage{amsmath}
\usepackage{amssymb}
\usepackage{amsthm}
\usepackage{bbm}
\usepackage{bm}
\usepackage{graphicx} 
\usepackage{subcaption}

\newtheorem{proposition}{Proposition}

\newcommand{\htheta}{\hat{\theta}}
\newcommand{\ttheta}{\tilde{\theta}}
\newcommand{\bx}{\bm{x}}
\newcommand{\by}{\bm{y}}

\newcommand{\bw}{\bm{w}}

\newcommand{\bbeta}{\bm{\beta}}
\newcommand{\ind}{\mathbb{I}}

\newcommand{\diff}{\Delta^2}
\newcommand{\hPhi}{\hat{\Phi}}

\colingfinalcopy 

\title{Finding Influential Instances for Distantly Supervised Relation Extraction}

\author{Zifeng Wang\textsuperscript{\rm 1,2}, Rui Wen\textsuperscript{\rm 3}, Xi Chen\textsuperscript{\rm 3}, Shao-Lun Huang\textsuperscript{\rm 2}, Ningyu Zhang\textsuperscript{\rm 4}, Yefeng Zheng\textsuperscript{\rm 3}\\ 
 \textsuperscript{\rm 1}University of Illinois Urbana Champaign  
 \textsuperscript{\rm 2}TBSI, Tsinghua University \\
 \textsuperscript{\rm 3}Jarvis Lab, Tencent
 \textsuperscript{\rm 4}Zhejiang University \\
Email: zifengw2@illinois.edu }

\date{}

\begin{document}
\maketitle
\begin{abstract}
Distant supervision (DS) is a strong way to expand the datasets for enhancing relation extraction (RE) models but often suffers from high label noise. Current works based on attention, reinforcement learning, or GAN are black-box models so they neither provide meaningful interpretation of sample selection in DS nor stability on different domains. On the contrary, this work proposes a novel model-agnostic instance sampling method for DS by influence function (IF), namely REIF. Our method identifies favorable/unfavorable instances in the bag based on IF, then does dynamic instance sampling. We design a fast influence sampling algorithm that reduces the computational complexity from $\mathcal{O}(mn)$ to $\mathcal{O}(1)$, with analyzing its robustness on the selected sampling function. Experiments show that by simply sampling the favorable instances during training, REIF is able to win over a series of baselines which have complicated architectures. We also demonstrate that REIF can support interpretable instance selection.
\end{abstract}

\section{Introduction}
To expand the training data for relation extraction (RE), distant supervision (DS) was proposed by \cite{mintz2009distant} who assumed that if two entities are related in existing KBs, then all sentences contain both of them express this relation. However, this heuristic inevitably suffers from wrong labels \cite{takamatsu2012reducing} and undermines model performance. For example, the sentence ``\emph{Bill Gates redefined the software industry, ... said Rob Glaser, a former Microsoft executive}" does not mention the relation \emph{founder} but is still treated as a positive training sample in DS. Dealing with noisy instances in DS has been a focus in RE. There are three main genres in the literature: (1) incorporating an attention module \cite{lin2016neural} to allocate confidence level among instances in the same bag; (2) using reinforcement learning \cite{qin2018robust} for instance selection; and (3) leveraging adversarial training \cite{wu2017adversarial} to enhance the RE model's robustness against noise. However, they are either black-box models thus unable to provide meaningful interpretation of sample selection or sensitive to datasets. More importantly, none of them is theoretically guaranteed to truely reduce the ``noise'' from the dataset.

In this work, we propose to leverage influence function (IF) to evaluate instance quality then do instance selection for DS. Influence function is a powerful tool drawn from robust statistics \cite{huber2004robust}. It is able to approximate the influence of a single data point on the whole model learned on the dataset. Creating to this merit, it has been successfully utilized for inspecting outliers \cite{boente2002influence} and denoising datasets \cite{wang2019less} based on shallow machine learning models, e.g., logistic regression. Although Koh \& Liang \shortcite{koh2017understanding} extends IF to interpreting deep networks, it is still elusive if it works for denoising datasets for deep networks. In this work, we develop the \textbf{R}elation \textbf{E}xtraction by \textbf{I}n\textbf{F}luence subsampling (REIF) framework, which aims for denoising DS for deep learning RE models. 

The high-level idea of REIF is shown by Fig. \ref{fig:0}. Each instance is assigned a quality measure $\phi$, from which its sampling probability is obtained via the sampling function $\pi$. Accordingly, the better an instance's quality is, the more likely it is picked during training. We will explain the operational meaning of $\phi$ in Section \ref{sec:3.2}. In a nutshell, the main contributions of this paper are 
\begin{itemize}
\item We develop a novel IF-based denosing framework for DS RE, namely REIF, for denoising RE by sampling favorable training instances.

\item An efficient implementation of REIF enables subsampling in $\mathcal{O}(1)$ complexity, instead of the $\mathcal{O}(mn)$ complexity without our implementation.

\item Empirical experiments show REIF's superiority over other baselines, and we identify its capability to support interpretable instance selection for RE by a case study.\footnote{Code is available in the supplementary materials.}
\end{itemize}

\begin{wrapfigure}{R}{0.5\textwidth}
\centering
  \includegraphics[width=0.99\linewidth]{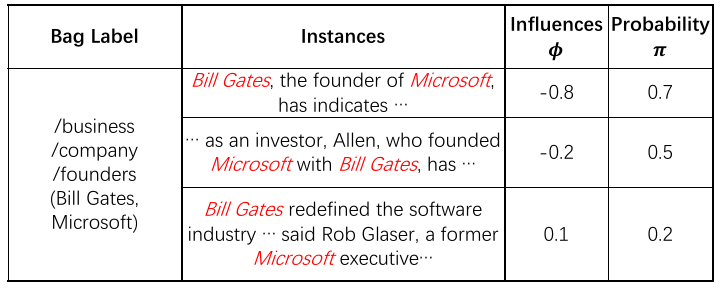}
\caption{Finding influential instances within a bag via subsampling based on the calculated probability $\pi$. Note that here \emph{negative} $\phi$ means a beneficial sample. \label{fig:0}}
		\vspace{-1.0em}
\end{wrapfigure}

\section{Related Work}
There are a series of works trying to address the noisy label difficulty in DS by multi-instance learning (MIL) \cite{hoffmann2011knowledge,riedel2010modeling,surdeanu2012multi}. MIL considers the training labels in \emph{bag} level instead of instance level. Each bag contains at least one instance with the labeled relation while the exact label of each instance is unknown. As MIL being proved effective in relation extraction, it was firstly introduced to neural relation extraction by Zeng et al. \shortcite{zeng2015distant}, where the piece-wise convolutional neural network (PCNN) was developed, and only one instance with the largest predicted probability was selected in each bag.

Later, attention \cite{lin2016neural,ZHOU2018240,jia2019arnor,yuan2019cross,ye2019distant,zhou2021self}, reinforcement learning \cite{feng2018reinforcement,yang2018distantly,qin2018robust,chen2021distant}, and adversarial training \cite{wu2017adversarial,qin2018dsgan,han2018denoising,shi2018genre} have been proposed for further improvement. However, above works usually require intense trials in fine-tuning of the hyper-parameters in practice, or are not interpretable to human-beings. In this work, we propose a model-agnostic and interpretable instance selection method via IF, which is easy-to-use for most DL models without many hyperparameters to choose.

\section{Methodology} \label{sec:prelim}
In this section, we elaborate on the major steps of REIF associated with the technical details and the theoretic foundation of measuring data quality by influences. Also, an analysis supporting our choice of sampling function is given.

\subsection{Relation Extraction by Influence Subsampling}
Our REIF is model-agnostic thus amenable to most DL models. Without loss of generality, we pick PCNN \cite{zeng2015distant} as the encoder for the input texts. The flowchart of our framework is shown in Fig. \ref{fig:1}. It includes three main parts: 1) backbone model and 2) instance selection.

\textbf{Backbone Model}. Inputs of the encoder are raw sentences represented by indices of words, e.g., a sentence $x_*$ with $l$ words $x_* = \{x_{*,1},\dots,x_{*,l} \}$. We transform them into dense real-valued representation vectors as $\bw_* = \{\bw_{*,1},\dots,\bw_{*,l}\}$, by concatenating the word embedding from $\bm{V} \in \mathbb{R}^{d^a \times |V|}$ (where $|V|$ denotes the size of the vocabulary and $d^a$ is the dimension of word embedding) and position embedding with dimension $d^p$ together. As there are two position embeddings, each word vector in $\bw$ has dimension $d^a + 2\times d^p$. Convolution layer processes the word representations as
\begin{equation} \label{eq:cnn}
\bx_* = {\rm CNN}(\bw_*).
\end{equation}
The CNN model receives representation vectors $\bw_*$ and outputs the processed feature vectors $\bx_* \in \mathbb{R}^{d\times l}$. The probability for relation prediction, taking $\bx_*$ as input, is given by
\begin{equation}
P(y=k | \bx_*) = \frac{\exp(\bbeta^{(k)\top}\bx_*)}{\sum_{k^\prime} \exp(\bbeta^{(k^\prime)\top}\bx_*) }, 
\end{equation} 
where $\bbeta = \{\bbeta^{(1)} \dots \bbeta^{(K)}\} \in \mathbb{R}^{d\times K}$ is the weight matrix of the last fully-connected layer; $K$ is the total number of relations.

\begin{figure}[t]
\centering
\includegraphics[width=0.95\textwidth]{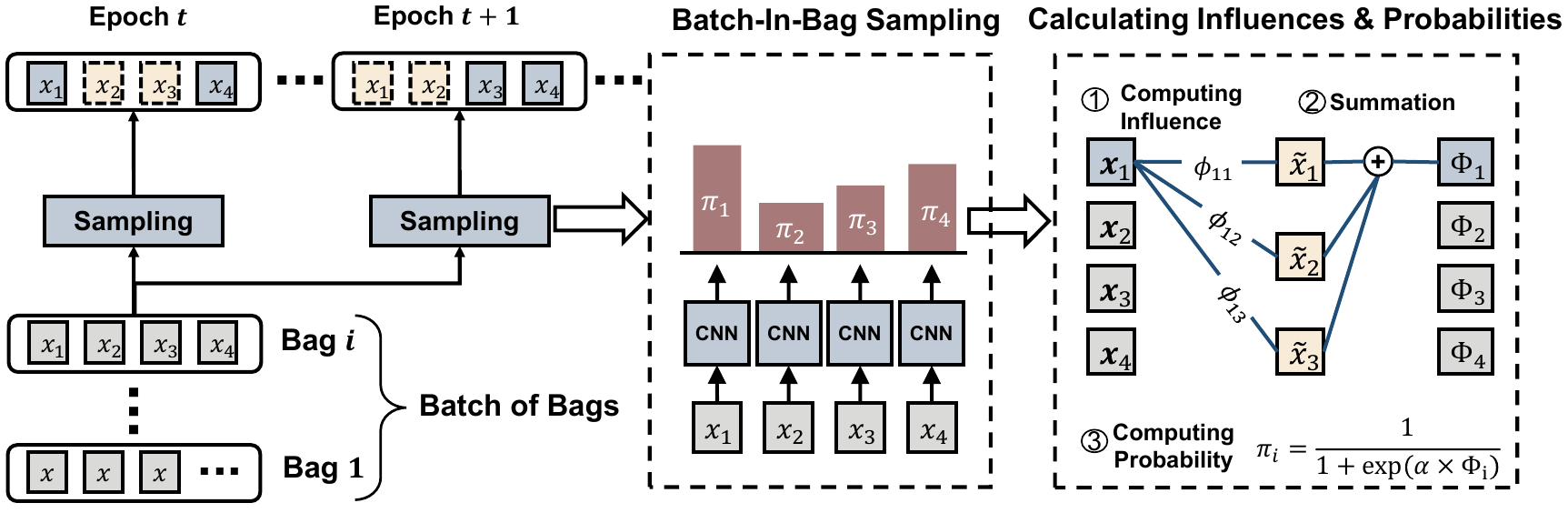}
\caption{The flowchart of the instance-level subsampling method, where $x$ is training sentence; $\tilde{x}$ is the validation sample; $\phi$ is the computed influence; and a dotted box means the instance is dropped after subsampling. \label{fig:1}}
\vskip -1.5em
\end{figure}

\textbf{Dynamic Instance Sampling}. One possible way to do sample selection by IF is \emph{post-hoc}, i.e., it first samples from the full training set, then retrains the model on the subsamples. However, we argue it is unsuitable for DS. In post-hoc sampling, all instances are gathered together, hence the subsamples are dominated by majority relations with lots of training instances, resulting in severe class imbalance. In an extreme case, minority relations may completely disappear after subsampling. 

On contrast, we propose dynamic instance sampling (DIS) which is executed within bags during training. Given a bag $X = \{x_1,\dots,x_n\}$ containing $n$ sentences, we try to sample a subset $X_{sub}$ with $|X_{sub}| < n$ from $X$. To this end, we calculate the influences $\Phi_i, \ \forall i=1,\dots,n$, and sampling probabilities $\pi_i$ are
\begin{equation} \label{eq:sigmoid}
\pi_i = \pi(\Phi_i) := \frac1{1 + \exp(\alpha \times \Phi_i)},
\end{equation}
where $\pi_i$ is the probability of $x_i$ being selected and $\alpha$ is a hyper-parameter. Consequently, the training objective function $J(\theta)$ is
\begin{equation}
J(\theta) = \frac1{|{X}_{sub}|} \sum_{x_i \in {X}_{sub}} \ell_i(\theta),
\end{equation}
where $\ell(\theta)$ is the abbreviation of loss function $\ell(x,y;\theta)$ for notation simplicity.

\subsection{Theoretic Foundation of Influence-based Sample Quality Measure} \label{sec:3.2}
The core step of REIF is to measure the instance influence $\Phi$. Intuitively, adverse instances, which cause model validation loss increasing, should be assigned low probability being sampled, and vice versa. We next present the property of $\Phi$ and substantiate this intuition in a rigorous way.

Consider a classification problem where we attempt to obtain a model $f_\theta: \mathcal{X} \to \mathcal{Y}$, which is parametrized by $\theta$, that can make prediction from an input space $\mathcal{X}$ (e.g., sentences) to an output space $\mathcal{Y}$ (e.g., relations). Given a set of training data $\{x_i\}_{i=1}^n$ and the corresponding labels $\{y_i\}_{i=1}^n$, the optimal $\hat{\theta}$ defined by
\begin{equation}
\htheta := \mathop{\arg\min}_{\theta \in \Theta} \frac1n \sum_{i=1}^n \ell_i(\theta).
\end{equation}
We evaluate the learned $f_{\htheta}$ on an additional validation set $\{(x_j^v, y^v_j)\}_{j=1}^m$ such as
\begin{equation} \label{eq:valloss}
L(\htheta) := \frac1m \sum_{j=1}^m \ell_j^v(\htheta)
\end{equation}
where $\ell_j^v(\hat{\theta})$ is the validation loss on $x_j^v$.

In order to quantitatively measure the $i$-th training sample's influence over model's validation loss, we can perturb the training loss $\ell_i(\theta)$ by a small $\epsilon$, then retrain a perturbed risk minimizer $\ttheta$ as
\begin{equation} \label{eq:perturbedtheta}
\ttheta := \mathop{\arg\min}_{\theta \in \Theta} \frac1n \sum_{i^\prime=1}^n \ell_{i^\prime}(\theta) + \epsilon \times \ell_{i}(\theta).
\end{equation}
As a result, we are able to compute the validation loss change of the validation sample $x^v_j$ by
\begin{equation} \label{eq:deltaj}
\delta_j(\epsilon) := \ell_j^v(\ttheta) - \ell_j^v(\htheta).
\end{equation}
It indicates to what extent $x_i$ \emph{influences} the prediction on $x^v_j$. If $\epsilon=-1/n$, according to Eq. \eqref{eq:perturbedtheta}, $x_i$'s loss $\ell_i(\theta)$ is actually \emph{removed} from the objective function. In this situation, $\delta_j(\epsilon) > 0$, i.e., $\ell_j^v(\ttheta) - \ell_j^v(\htheta) > 0$, implies that removing $x_i$ causes the validation loss on $x^v_j$ increasing, i.e.,
\begin{equation}
\delta_j\left(-\frac1n\right)>0  \to x_i \ \text{is good for} \ x_j^v.
\end{equation}
The influence function $\phi_{i,j} := \phi(x_i, x^v_j; \htheta)$ linearly approximate $\delta_j(\epsilon)$ by
\begin{equation} \label{eq:influence}
\delta_j(\epsilon) = \ell_j^v(\ttheta) - \ell_j^v(\htheta)  \simeq \epsilon \times  \phi_{i,j},
\end{equation}
where the closed-form expression of $\phi$ is given in \cite{koh2017understanding} as
\begin{equation} \label{eq:closedformphi}
\phi_{i,j} := - \nabla_{\theta} \ell_j^v(\htheta)^{\top} H_{\htheta}^{-1} \nabla_{\theta} \ell_i(\htheta)
\end{equation}
and $H_{\htheta} := \frac1n \sum_{i=1}^n \nabla_{\theta}^2 \ell_i(\htheta)$ is the Hessian matrix. 

In short, by Eq. \eqref{eq:influence}, $\delta_j(-1/n) > 0$ is equivalent to $\phi_{i,j} < 0$. We can compute $x_i$'s influence over the whole validation set by summation
\begin{equation} \label{eq:phiidef}
\Phi_i  = \sum_{j=1}^m \phi_{i,j} = - \sum_{j=1}^m \nabla_\theta \ell_j^{v \top}(\htheta)H_{\htheta}^{-1} \nabla_{\theta} \ell_i(\htheta).
\end{equation}
Now, $\Phi_i < 0$ implies that $x_i$ is \emph{good} for the whole validation set. Also, if $\Phi_i$ is smaller, then $x_i$ is more likely to be a favorable sample, and vice versa.

\subsection{On Robustness of Sampling Functions}
With the influence measure $\Phi$, it seems that we can simply drop all \emph{unfavorable} samples that have $\Phi>0$. However, we argue that using $0$ as the threshold usually results in failure to the out-of-sample test, due to its sensitivity to distribution shift. Instead, we take the measure of \textit{probabilistic} sampling by designing a sampling function $\pi(\Phi) \in [0,1]$. We give the reason of this choice based on the deviation of the induced validation loss by inaccurate estimate of influence. Let's denote the validation loss with inaccurate influence by $\ell^v(\ttheta;\hat{\Phi})$, thus
\begin{equation}
\diff( L) := \frac1m \sum_{j=1}^m(\ell_j^v(\ttheta;\hat{\Phi})-\ell_j^v(\ttheta))^2
\end{equation}
indicates the robustness of the model under $\hat{\Phi}$. We then give the following proposition on $\diff(L)$ with respect to sampling function $\pi$. Proof can be found in Appendix \ref{app:1}.

\begin{proposition}[Robustness of Probabilistic Sampling under Inaccurate Influence] \label{theo:var}
Let $\pi^{\prime}(\Phi_i)$ be the derivative of $\pi(\cdot)$ function when taking $\Phi_i$ as its input, we have
\begin{equation}\label{eq:propvar}
 \sup_{\Phi,\hPhi}   \diff(L)  = \gamma  \sum_{i=1}^n (\pi(\hPhi_i) - \pi(\Phi_i))^2 \sum_{j=1}^m \phi_{i,j}^2 
\simeq \gamma \sum_{i=1}^n \left( (\hPhi_i - \Phi_i)\pi^{\prime}(\Phi_i)  \right)^2 \sum_{j=1}^m \phi_{i,j}^2
\end{equation}
where $\gamma$ is a constant.
\end{proposition}

It can be viewed that $\diff(L)$ is controlled by the derivative of sampling function $\pi^{\prime}(\Phi)$. For the sigmoid sampling in Eq. \eqref{eq:sigmoid}, it is easy to derive that
\begin{equation}
\pi^{\prime}(\Phi) = -\alpha \pi(\Phi) (1 - \pi(\Phi)),
\end{equation}
which means $\max |\pi^{\prime}(\Phi)| = \frac14 \alpha$ when $\Phi = 0$. $\diff(L)$ is hence controlled by the hyper-parameter $\alpha$. When $|\Phi|$ increases, $|\pi^\prime(\Phi)|$ reduces sharply, which ensures the variance's upper bound being tight all the time. By contrast, in deterministic sampling, $\diff(L)$ is sensitive to inaccurate $\hPhi$ because it is ``hard", or more rigorously, because $\diff(L)$ is probably large due to large $|\pi(\Phi)-\pi(\hPhi)|$ caused by an improper dropout threshold.

\section{Efficient Implementation} \label{sec:method}
Recap Eq. \eqref{eq:phiidef}, computing $\Phi_i$ requires $\phi_{i,j}$ in Eq. \eqref{eq:closedformphi} for $j = 1, \dots, m$ on all validation samples. As a result, the computation of all $\{\Phi_i\}_{i=1}^n$ has $\mathcal{O}(mn)$ time complexity. Moreover, for DNNs with massive parameters, computing the layer-wise gradients $\nabla_\theta \ell(\theta)$ is intractable. These limitations prevent the use of IF from DL RE models. To address it, we here propose a rather efficient implementation of REIF. We demonstrate how to reduce the complexity of calculating influences from $\mathcal{O}(mn)$ to $\mathcal{O}(n)$, then to $\mathcal{O}(1)$. In addition, we show how to compute the influence function by stochastic estimation.

\subsection{Computing Influences in Linear Time}
 We argue that in Eq. \eqref{eq:phiidef}, it is unnecessary to calculate $\phi_{i,j}$ separately, since here we only care about their summations. Specifically, since the summation is only related to the subscript $j$, we can cast it to
\begin{equation}\label{eq:lossagg}
\Phi_i  = - \nabla_\theta \ell_i^{\top}(\htheta) H_{\htheta}^{-1} \sum_{j=1}^m \nabla_\theta \ell_j^v(\htheta) 
= - \nabla_\theta \ell_i^{\top}(\htheta) H_{\htheta}^{-1}  \nabla_\theta \sum_{j=1}^m \ell_j^v(\htheta) 
=  - m \nabla_\theta \ell_i^{\top}(\htheta) H_{\htheta}^{-1} \nabla_\theta L(\htheta),
\end{equation}
where $L(\htheta)$ comes from Eq. \eqref{eq:valloss}. By this derivation, we can calculate $L(\htheta)$ rather than all $l_j(\htheta)$, then take derivative of $L(\htheta)$. Since $L(\htheta)$ only needs to be calculated once and it is shared in calculating all $\Phi_i$s, this process only requires $\mathcal{O}(n)$ time, without loss of accuracy.

\subsection{Linear Approximation for $\mathcal{O}(1)$ Complexity}
$\nabla_\theta \ell(\htheta)$ in Eq. \eqref{eq:lossagg} usually has complicated expression when $f_\theta(\cdot)$ is a neural network, hence the previous works implemented it by the auto-grad systems like TensorFlow \cite{abadi2016tensorflow} and PyTorch \cite{paszke2019pytorch}. However, when the number of alternative training instances is large, even $\mathcal{O}(n)$ is not satisfactory enough, because additional differential operations need to be done on each $\ell_i(\htheta)$ sequentially. Moreover, when faced with complex neural networks with massive parameters, computing the Hessian matrix $H_{\htheta}$ and its inversion is intractable. Considering these issues, we propose a linear approximation approach to reduce the complexity to $\mathcal{O}(1)$, and avoid operating on all parameters of the neural network. 

Suppose the cross entropy loss function is used:
\begin{equation}
\ell(\theta) = - \sum_{k=1}^K \ind \{y = k\} \log P(y=k | x; \theta)
\end{equation}
where $\ind(\cdot)$ is an indicator function. Let $\by, \hat{\by} \in \mathbb{R}^K$ be the one-hot label vector, e.g., $(1,0,0)^{\top}$, and prediction vector, e.g., $(0.8,0.1,0.1)^{\top}$, respectively. We replace $\nabla_\theta \ell(\theta)$ in Eq. \eqref{eq:closedformphi} with the derivatives on $\bbeta$ (the weight of the last fully-connected layer):
\begin{equation} \label{eq:nablaltr}
\nabla_\theta \ell(\theta) \Rightarrow \nabla_{\bbeta} \ell(\theta) = (\hat{\by} - \by) \bx^{\top} \in \mathbb{R}^{d \times K}
\end{equation}
where $\bx$ is the input of the last fully-connected layer. This closed-form expression allows computing batch gradients in $\mathcal{O}(1)$ time. Although the calculated influence might be inaccurate, it is still reliable for measuring instances' \emph{relative quality} in general. We will validate this claim in our experiments.

\begin{figure}[t]
\centering
\begin{minipage}[t]{0.40\textwidth}
\begin{flushleft}
\begin{subfigure}[t]{0.95\textwidth}
\includegraphics[width=0.99\linewidth]{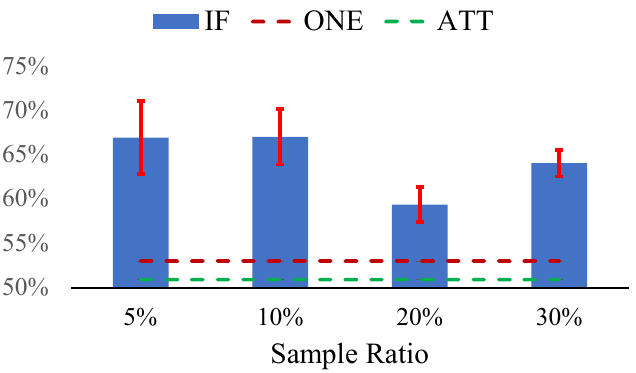}
    \caption{NYT-SMALL}
\end{subfigure}
\begin{subfigure}[t]{0.95\textwidth}
    \includegraphics[width=0.99\linewidth]{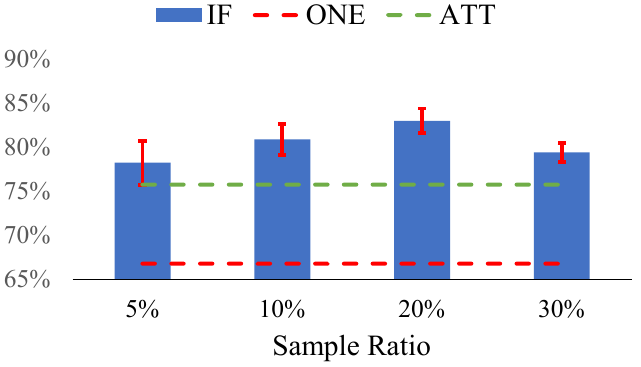}
    \caption{NYT-LARGE}
\end{subfigure}
\caption{Mean P@N (average of P@100/200/300) varies with sampling ratio of REIF (IF) method. Red bar represents standard error by 5 times repeat experiments.  \label{fig:patn_ratio}}
\end{flushleft}
\end{minipage}
\hspace{.3em}
\begin{minipage}[t]{0.40\textwidth}
\begin{flushright}
\begin{subfigure}[t]{0.95\textwidth}
\includegraphics[width=0.99\linewidth]{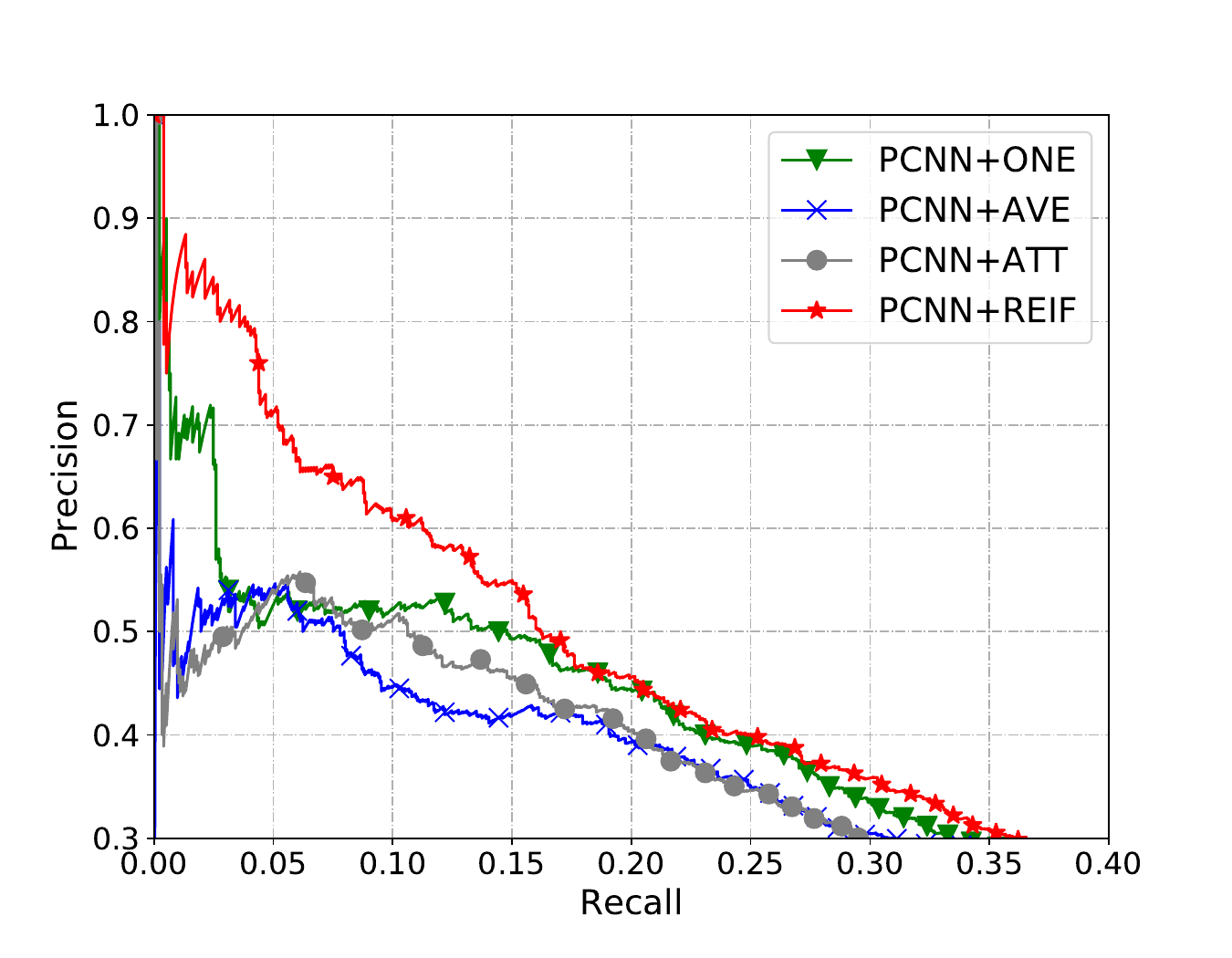}
    \caption{NYT-SMALL}
\end{subfigure}
\begin{subfigure}[t]{0.95\textwidth}
    \includegraphics[width=0.99\linewidth]{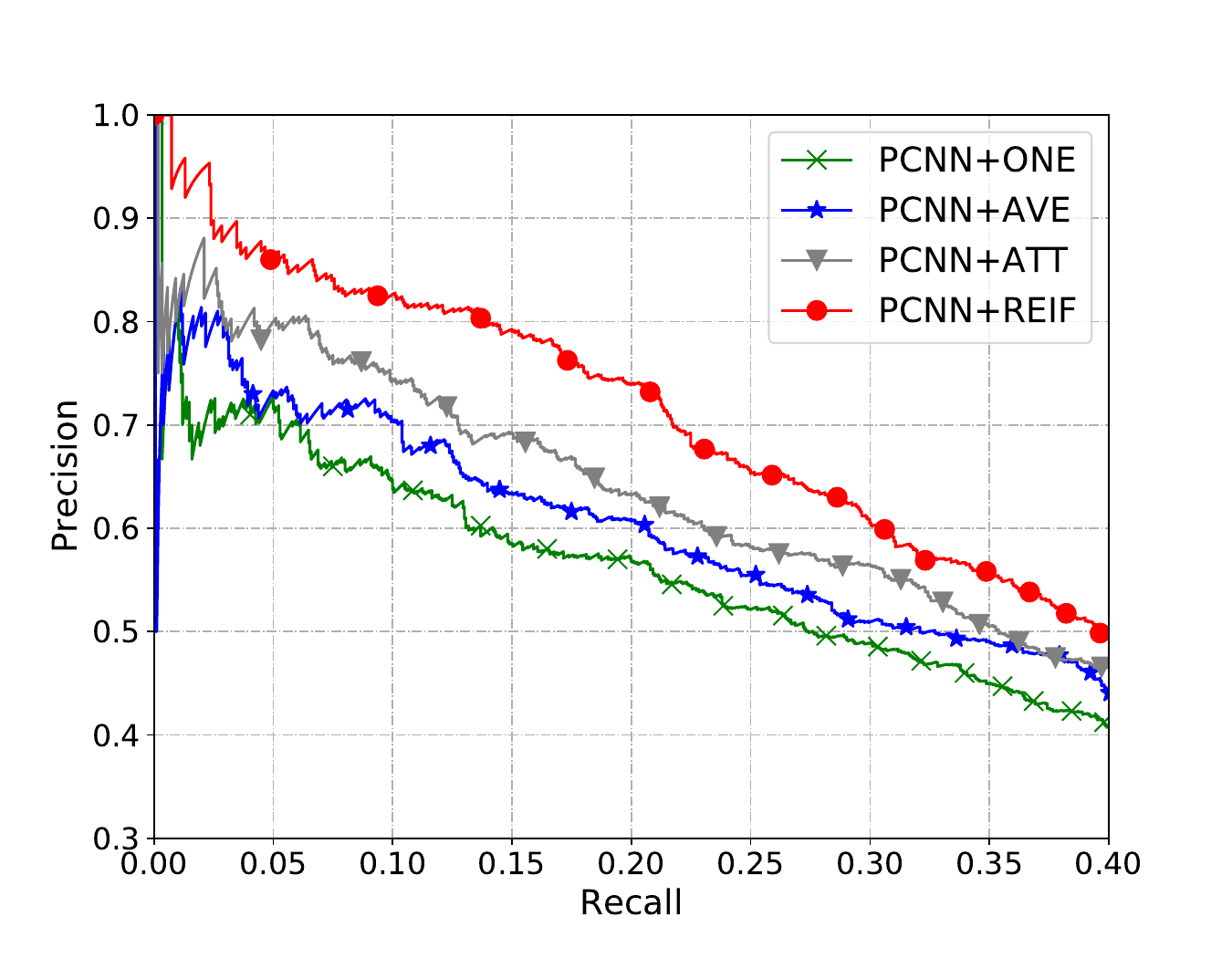}
    \caption{NYT-LARGE}
\end{subfigure}
    \caption{Aggregated precision-recall (P-R) curves obtained by PCNN+ONE, PCNN+AVE, PCNN+ATT, and the proposed PCNN+REIF on NYT-SMALL (left) and NYT-LARGE (right) datasets.\label{fig:res}}
\end{flushright}
\end{minipage}
\vspace{-1em}
\end{figure}

\subsection{Algorithm}
Algorithm \ref{alg:1} shows the details of REIF, please refer to Appendix \ref{app:alg}. It has two hyper-parameters: the sampling ratio $r$ and the sigmoid sampling parameter $\alpha$. The optimal value of $r$ depends on quality of the dataset, since the higher quality it is, the more favorable instances it might have. Keeping $\alpha=1$ is satisfactory in most scenarios.

In particular, on the line \#14 of Algorithm \ref{alg:1}, we compute the product between the inverse Hessian matrix and a gradient vector via the stochastic estimation procedure by Koh \& Liang \cite{koh2017understanding}. Denoting the vector $\nabla_\theta L(\htheta)$ by $v$, it first initializes the approximate inverse Hessian-Vector-Product (HVP) by $\tilde{H}_0^{-1}v \gets v$, then repeatedly samples $n_b$ training instances and updates as
\begin{equation} \label{eq:stochasticest}
\tilde{H}_t^{-1}v \gets v + \left(I - \frac1{n_b}\sum \nabla_\theta^2 \ell(\htheta) \right)\tilde{H}_{t-1}^{-1}v
\end{equation}
until $\tilde{H}_t^{-1}v$ converges. In our algorithm, we only need to do this once after each epoch, to get the precomputed inverse HVP $s = H_{\htheta}^{-1} \nabla_\theta L(\htheta)$. Therefore, during training, we directly compute $\nabla_\theta \ell_i(\htheta)$ for each instance according to Eq. \eqref{eq:nablaltr}, then multiply it with the precomputed $s$.

\section{Experiments}
We concentrate on the following research questions:

\textbf{RQ1}. How does our REIF perform as compared with classical baselines?

\textbf{RQ2}. How does the sampling ratio $r$ influence the performance of the REIF?

\textbf{RQ3}. Does the sigmoid function lead to more robust sampling than the deterministic sampling?

\textbf{RQ4}. How does the proposed dynamic instance sampling perform compared with the post-hoc sampling using IF?

\begin{table}[t]
  \centering
  \caption{P@N for relation extraction results, on NYT-SMALL and NYT-LARGE, where the best ones are in bold. }
    \begin{tabular}{|l|c|c|c|c|c|c|c|c|}
    \hline
    Dataset & \multicolumn{4}{c|}{{\small NYT-SMALL}}   & \multicolumn{4}{c|}{{\small NYT-LARGE}} \\
\hline
    P@N (\%) & 100   & 200   & 300   & Mean  & 100   & 200   & 300   & Mean \\
\hline
    PCNN + ONE & 54.0 & 52.7 & 52.2 & 53.0 & 70.4 & 66.4 & 63.6 & 66.8 \\
\hline
    PCNN + AVE & 52.7 & 50.8 & 47.3 & 50.3 & 73.0 & 71.2 & 67.8 & 70.6 \\
\hline
    PCNN + ATT & 52.7 & 50.7 & 49.5 & 50.9 & 79.7 & 76.0 & 71.6 & 75.8 \\
\hline
    PCNN + REIF (Proposed) & \textbf{75.2} & \textbf{65.1} & \textbf{60.8} & \textbf{67.0} & \textbf{86.4} & \textbf{82.5} & \textbf{80.3} & \textbf{83.1} \\
\hline
    \end{tabular}%
  \label{tab:patn}%
\end{table}%

\begin{table}[t]
  \centering
  \caption{Prevision (\%) of various DS methods using PCNN as backbones / other DS methods for different recalls (0.1, 0.2, 0.3) on NYT-LARGE. The results of cited methods are drawn from their papers, and the best are in bold.}
\setlength{\tabcolsep}{4mm}{
    \begin{tabular}{l|c|c|c|c}
    \hline
    PCNN & 0.1   & 0.2   & 0.3   & Mean \\
    \hline
    +ONE & 64.7  & 57.1  & 48.9  & 56.9 \\
    +ATT & 74.3  & 63.3  & 56.5  & 64.7 \\
    +ONE+soft-label \cite{liu2017soft} & 71.6  & 62.5  & 54.1  & 62.7 \\
    +ATT+soft-label \cite{liu2017soft} & 75.1  & 67.5  & 55.8  & 66.1 \\
    +ONE+DSGAN \cite{qin2018dsgan}& 65.5  & 57.2  & 50.0  & 57.6 \\
    +ATT+DSGAN  \cite{qin2018dsgan}& 70.5  & 62.2  & 53.3  & 62.0 \\
    +PE+REINF \cite{zeng2018large} & 70.1  & 66.2  & 56.1  & 64.1 \\
    +ONE+RL \cite{qin2018robust} &  66.7 & 56.1  & 48.3  & 64.1 \\
    +ATT+RL \cite{qin2018robust} & 68.3  &  60.0  &  52.2 & 60.2 \\
    +ONE+ADV \cite{wu2017adversarial} & 71.7  & 58.9  & 51.1  & 60.6 \\
    +ONE+AN \cite{han2018denoising} & {80.3}  & 70.2  & 60.3  & 70.3 \\
    +ATT-RA+BAG-ATT \cite{ye2019distant} & 78.8 & 68.9 & 62.1 & 69.9 \\
    +SATT \cite{zhou2021self} & 78.2 & 69.1 & 59.5 & 68.9 \\
    \hline
    DISTRE \cite{alt2019fine} & 65.2  & 64.4 & 60.9 & 63.5\\
    RedSandT \cite{christou2021improving} & 73.1 & 67.3 & 58.0 & 66.1
    \\
    Trans-SA \cite{xiao2022hybrid} & 74.1 & 67.2  & 57.9  & 66.4 \\
\hline
    PCNN+\textbf{REIF} (Ours) & \textbf{82.6}  & \textbf{73.9}  & {60.9}  & \textbf{72.5} \\
    \hline
    \end{tabular}%
}
  \label{tab:p_r_baseline}%
\vskip -1em
\end{table}%

\begin{wrapfigure}{R}{0.5\textwidth}
\centering
  \includegraphics[width=0.99\linewidth]{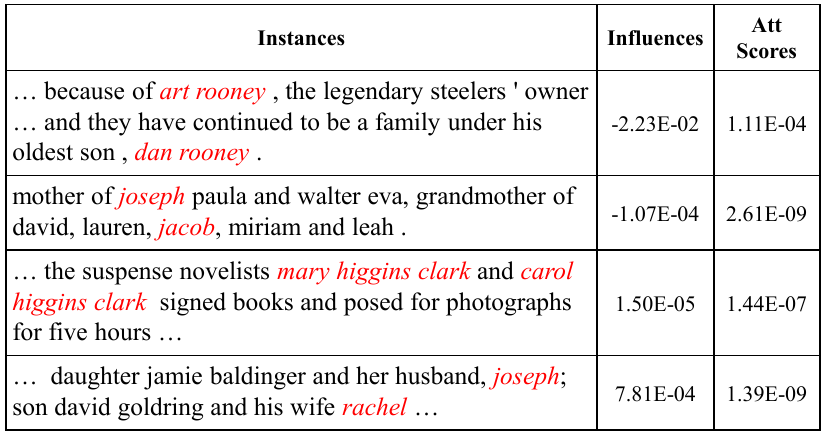}
\caption{Examples of influences calculated with the relation \emph{children}, on NYT-LARGE. The words in bold are entities. The \emph{Att Scores} \cite{lin2016neural} are standardized into $[0,1]$ by softmax, and \emph{Influence} is the smaller the better. \label{fig:case}}
		\vspace{-1.0em}
\end{wrapfigure}

\subsection{Datasets}
In our experiments, we use two versions of widely used NYT datasets, the NYT-SMALL and NYT-LARGE. The small version is released in \cite{riedel2010modeling}, by aligning Freebase with the New York Times corpus. In particular, we use the filtered version of the NYT-SMALL released by \cite{zeng2015distant}. The large version was released by \cite{lin2016neural}. Data statistics can be found in Appendix \ref{app:2}.

\begin{wrapfigure}{R}{0.4\textwidth}
\centering
\begin{subfigure}[t]{0.99\linewidth}
\includegraphics[width=0.99\linewidth]{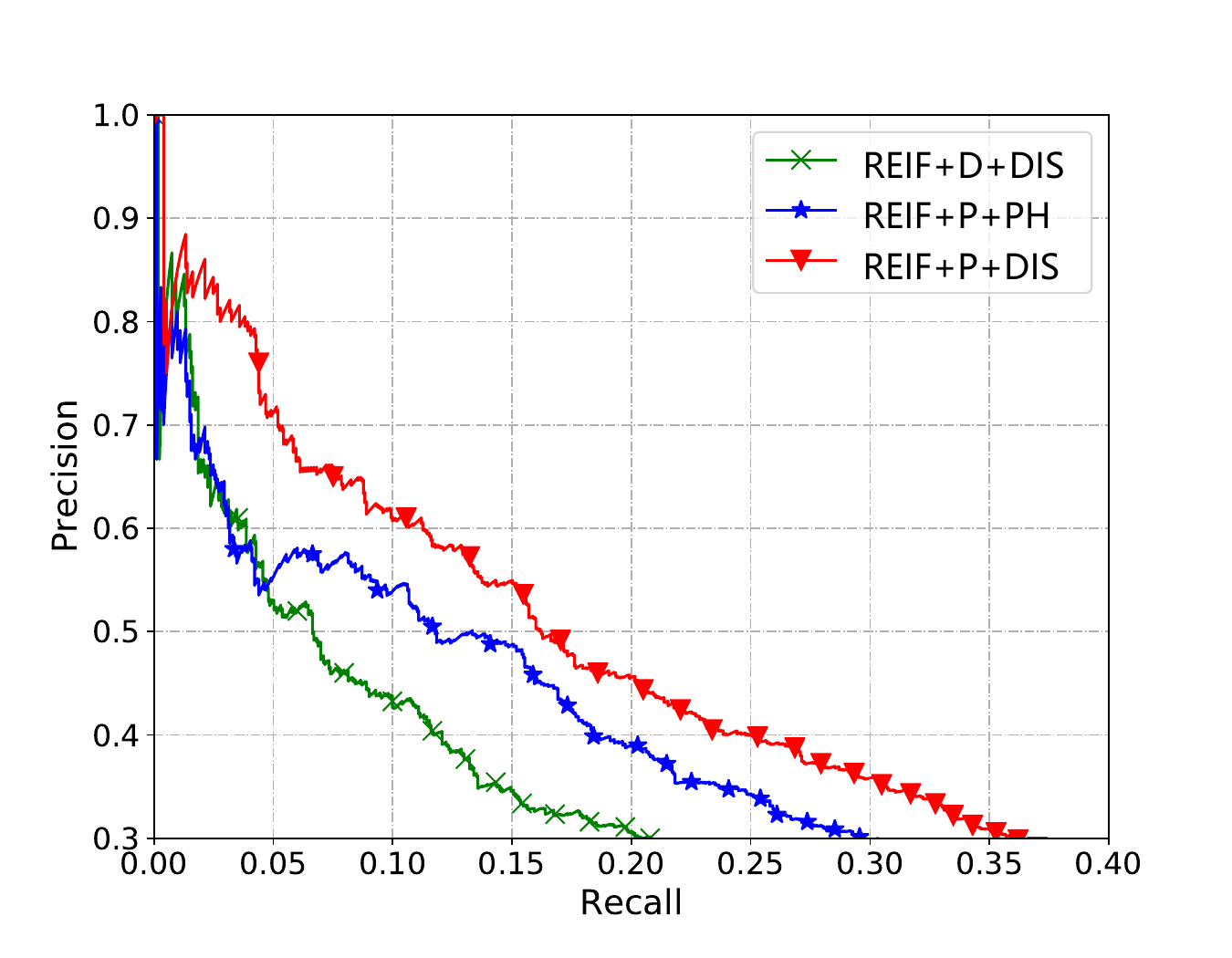}
    \caption{NYT-SMALL}
\end{subfigure}
\begin{subfigure}[t]{0.99\linewidth}
    \includegraphics[width=0.99\linewidth]{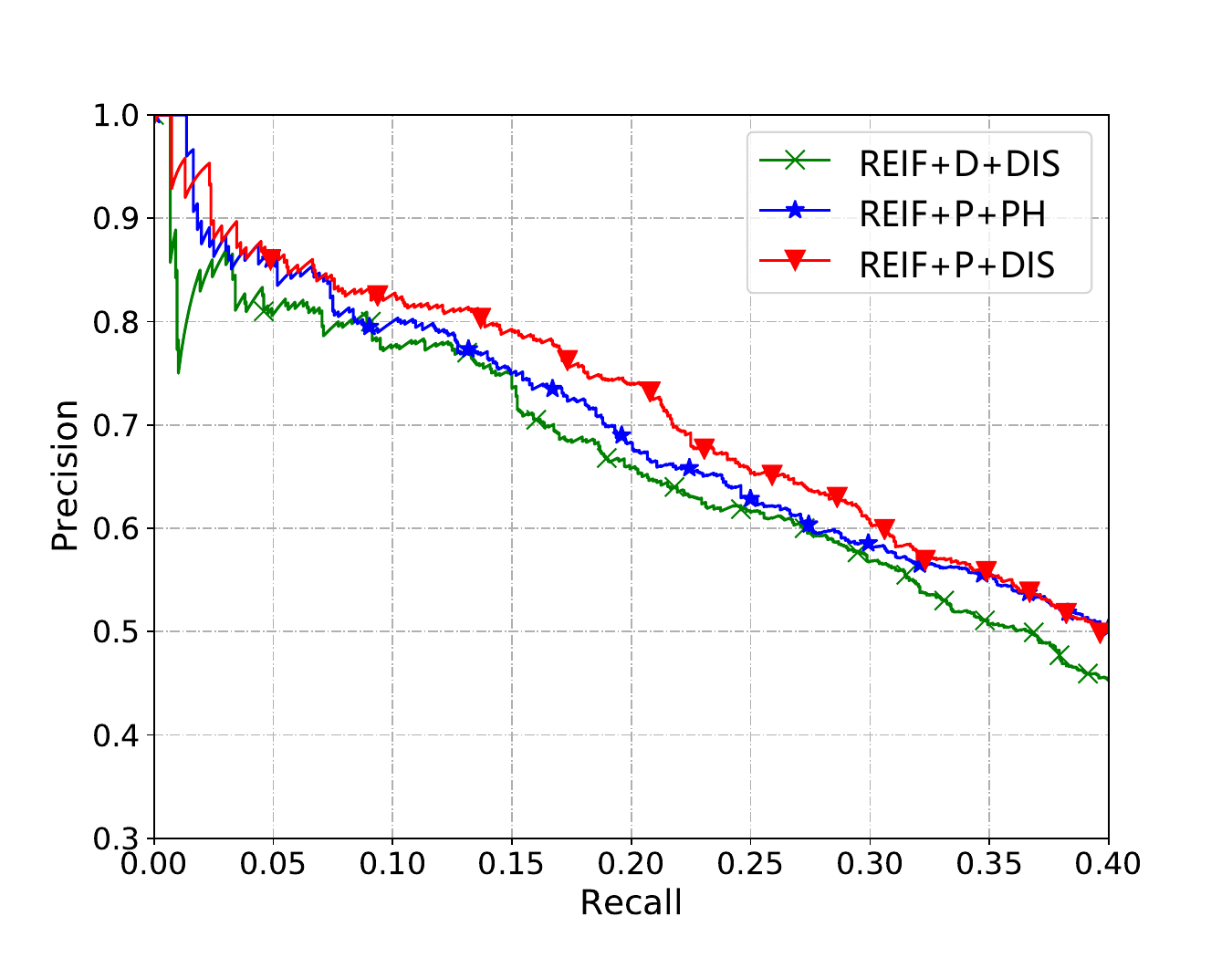}
    \caption{NYT-LARGE}
\end{subfigure}
\caption{Precision-recall curve of compared REIF variants, where the REIF+P+DIS is the REIF with probabilistic sigmoid sampling and dynamic sampling, +D means deterministic sampling and +PH means post-hoc sampling. \label{fig:5}}
\vspace{-1em}
\end{wrapfigure}

\subsection{Experimental Setups}
We pick PCNN (PCNN+ONE) \cite{zeng2015distant} as the backbone in our experiments, and include several baselines for comparison: the attention-based PCNN (PCNN+ATT) and the naive average method (PCNN+AVE) \cite{lin2016neural}. Note that our REIF method is model-agnostic, hence it is applicable for other deep learning based backbones as well, e.g., CNN and RNN. Setups of models can be found in Appendix \ref{app:3}.

We sample a clean validation set from training set by a rule-based approach used in \cite{jia2019arnor}, in order to obtain the inverse HVP required for calculating influences. The details of its establishment and discussions of this validation set can be found in Appendix \ref{app:4}. During subsampling, we set $\alpha=1$ and $r \in \{5\%,10\%,20\%,30\%\}$\footnote{The ceiling function is used for rounding.} for our REIF.

\subsection{Effects of Influence Subsampling (RQ1)}
Fig. \ref{fig:res} shows the precision-recall curve in held-out evaluation of ONE, AVE, ATT, and our REIF, and Table \ref{tab:patn} illustrates the corresponding P@N of all methods. Our REIF performs the best among all methods. In details, on NYT-SMALL, our REIF improves $14\%$ over ONE, and $16.1\%$ over ATT; on NYT-LARGE, the improvements are $14.1\%$ and $5.1\%$, respectively, in terms of the mean P@N. Specifically, REIF only leverages part of instances during training, while ATT involves all instances but performs badly on NYT-SMALL, and ONE only picks one instance per bag. It means that neither picking too many nor too few instances gains satisfactory performance in distant supervision. On contrast, our REIF can detect and pick those favorable ones from the noisy dataset, thus achieving a better model. In distant supervision, our method is effective for achieving nice trade-off between efficiency and effectiveness. Moreover, we compare our method with many DS baselines, including adversarial training, reinforcement learning, attention, and GAN based methods, using the reported results. As shown in Table \ref{tab:p_r_baseline}, REIF still is superior.

\subsection{Effects of Sampling Ratio (RQ2)}
We evaluate the performance of REIF with respect to different $r$ by repeat experiments. Results are reported in Fig. \ref{fig:patn_ratio}. REIF keeps stable when sampling ratio ranges from $5\%$ to $30\%$, such that adding more instances does not make much difference, which might be due to high noise in the NYT dataset, i.e., focusing on those favorable instances is enough for training a satisfactory RE model.

\subsection{Effects of Sigmoid Sampling \& Dynamic Sampling (RQ3, RQ4)}
Our REIF is engaged with the proposed probabilistic sigmoid sampling and DIS, namely REIF+P+DIS. We would like to validate these two techniques compared with the deterministic sampling (REIF+D+DIS), and the post-hoc sampling (REIF+P+PH). Our main observations from Fig. \ref{fig:5} are as follows:

(1) The probabilistic sigmoid sampling is crucial for robust subsampling, as the REIF+D+DIS performs the worst in both datasets. As mentioned in Proposition \ref{theo:var}, drawbacks of REIF+D mainly come from the inaccurate estimate of influence $\hPhi$, due to the non-convexity of neural networks and the use of linear approximations. That is, we could not determine the instances that have $\hPhi$ around the threshold with very high confidence, e.g., deterministic ranking and selecting, since this causes high variance of the resulting test loss, as indicated by Eq. \eqref{eq:propvar}. By contrast, we should assign them similar probabilities to be sampled, as done in REIF+P, to avoid sharp variation of the test loss caused by inaccurate influences in deterministic selection.

(2) Our dynamic sampling method generally performs better than post-hoc sampling in DS, especially on the tail instances. When recall is high, REIF+DIS performs better on the minor relations, thus has higher precision than REIF+PH. In DIS, more minor relation instances are maintained, which facilitates the model's capacity of mining minor relation instances. Considering efficiency and the overall effectiveness, we shall prefer DIS in practice.

\begin{table}[t]
  \centering
  \caption{Precision values for the top 100, 200 and 500 via manual evaluation. Avg denotes the average of the former three columns. Best ones are in bold.}
\setlength{\tabcolsep}{1mm}{
    \begin{tabular}{c|cccc}
    \hline
    Accuracy (\%) & Top 100 & Top 200 & Top 500 & Avg \\
    \hline
    Mintz & 77  & 71  & 55  & 67.7 \\
    MultiR & 83  & 74  & 49  & 68.7 \\
    MIML  & 85  & 75  & 61  & 73.7 \\
    PCNN+ONE & 86  & 80   & 69  & 78.3 \\
    APCNN & 87  & 82  & 72  & 80.3 \\
\hline
    PCNN+ATT & 86  & 81  & 70   & 79.0 \\
    PCNN+REIF & \textbf{88}  & \textbf{84}  & \textbf{76}  & \textbf{82.7} \\
    \hline
    \end{tabular}%
}
  \label{tab:manual_eval}%
\end{table}%

\section{Manual Evaluation \& Case Study}
Held-out evaluation usually suffers from false negative examples in Freebase \cite{zeng2015distant}. To further check our method, we perform manual evaluation by choosing the entity pairs which are labeled as ``NA'' but predicted a relation (not ``NA'') with high confidence. The top-$k$ precisions are reported in Table \ref{tab:manual_eval}, where the results of Mintz \cite{mintz2009distant}, MultiR \cite{hoffmann2011knowledge}, MIML \cite{surdeanu2012multi}, PCNN+ONE \cite{zeng2015distant} and APCNN \cite{ji2017distant} are drawn from their papers. It could be seen our method outperforms baselines in extracting new facts from the false negative examples.

Fig. \ref{fig:case} reports an example of calculating influences that support instance selection. Picking a relation \emph{children} as the example, influences and attention scores \cite{lin2016neural} are computed, from which we can identify that the influences quantitatively measure their individual quality. Recall in Section \ref{sec:3.2} that the smaller influences indicate better data quality. The first and the last instances are clearly right and wrong, respectively, in terms of indicating the relation \emph{children} between their entities. By contrast, the second one tends to be right because it implies that \emph{Joseph} is the parent of \emph{Jacob}. Although two entities in the third instance are very similar, no evidence shows they are relatives. Therefore, sampling probabilities can be obtained via these influences for the further subsampling process. 

\section{Conclusion \& Discussion}
In this work, we proposed an efficient subsampling scheme to find the influential instances for DS, namely REIF. Our method is model-agnostic, therefore it can be engaged in the majority of RE models. REIF can be generalized to other tasks which also confront noisy data. For instance, in other weak supervision scenarios such as active learning, our method can be an effective approach to build data pipeline from data quality measure to data selection. We leave this as our future work.

\bibliographystyle{coling}
\bibliography{coling2020}
\appendix
\setcounter{table}{0}
\setcounter{theorem}{0}
\setcounter{lemma}{0}
\setcounter{equation}{0}
\setcounter{proposition}{0}
\setcounter{figure}{0}
\numberwithin{equation}{section}
\numberwithin{lemma}{section}
\numberwithin{definition}{section}
\numberwithin{figure}{section}

\clearpage

\section{Proof of Proposition 1}\label{app:1}
\begin{proposition}[Robustness of Probabilistic Sampling under Inaccurate Influence] \label{theo:var}
Let $\pi^{\prime}(\Phi_i)$ be the derivative of $\pi(\cdot)$ function when taking $\Phi_i$ as its input, we have
\begin{equation}\label{eq:propres}
\begin{aligned}
 & \sup_{\Phi,\hPhi}   \diff(L)  = \gamma  \sum_{i=1}^n (\pi(\hPhi_i) - \pi(\Phi_i))^2 \sum_{j=1}^m \phi_{i,j}^2 \\
& \simeq \gamma \sum_{i=1}^n \left( (\hPhi_i - \Phi_i)\pi^{\prime}(\Phi_i)  \right)^2 \sum_{j=1}^m \phi_{i,j}^2
\end{aligned}
\end{equation}
where $\gamma$ is a constant.
\end{proposition}
\begin{proof}
\begin{align}
& \diff(L) \propto \sum_{j=1}^m(\ell^v_j(\ttheta;\hat{\Phi})-\ell_j^v(\ttheta))^2 \\
& = \sum_{j=1}^m (\ell_j^v(\ttheta;\hat{\Phi}) - \ell_j^v(\htheta) + \ell_j(\htheta) - \ell_j^v(\ttheta))^2 \\
& \propto \sum_{j=1}^m \left( \sum_{i=1}^n \pi(\hPhi_i)\phi_{i,j} - \pi(\Phi_i)\phi_{i,j} \right )^2 \label{eq:refertoless} \\
& \leq \sum_{i=1}^n (\pi(\hPhi_i) - \pi(\Phi_i))^2 \sum_{j=1}^m \phi_{i,j}^2
\end{align}
Eq. \eqref{eq:refertoless} is obtained by definition of probabilistic subsampling because
\begin{equation}
\begin{split}
\ell_j^v(\ttheta) - \ell_j^v(\htheta) &\simeq \sum_{i=1}^n \epsilon_i \phi_{i,j} \\
&\propto \sum_{i=1}^n \pi(\Phi_i)\phi_{i,j}.
\end{split}
\end{equation}
Details can be referred to \cite{wang2019less}.
Taking linear Taylor expansion of the $\pi(\hPhi_i) - \pi(\Phi_i)$ at the last line yields the final result.
\end{proof}

\section{Algorithm} \label{app:alg}
\begin{algorithm}[t] 
\caption{Finding Influential Instances for DS on RE by Influence Subsampling. \label{alg:1}}
\begin{algorithmic}[1] 
\Require Training and validation data $\mathcal{D}_{tr}, \mathcal{D}_{va}$; Hyper-parameters: $r$ and $\alpha$;
\For{epoch $t = 1 \to T$}
\Repeat 
\State Initialize the selected instances set ${X}_{sub} = \emptyset$;
\State Sequentially sample a batch of bags $\{{X}_1, \dots, {X}_{B}\}$ from $\mathcal{D}_{tr}$;
\For{bag $b = 1 \to B$}
\State Obtain instance-level loss as $\vec{\ell} \gets (\ell_1(\htheta_{t}),\dots,\ell_{|{X}_b|}(\htheta_{t}))^{\top}$;
\State Compute influences $\Phi_i \gets s_t^{\top}\nabla_\theta \ell_i(\htheta_t) \ \forall i=1,\dots,|{X}_b|$;
\State Compute sampling probability $\pi_i \gets 1/(1 + \exp(\alpha \times \Phi_i)) \ \forall i$;
\State Sample $r\times |X_b|$ instances from ${X}_b$ to get $\tilde{{X}}_b$, and ${X}_{sub} \gets {X}_{sub} \cup \tilde{X}_b$;
\EndFor
\State Update $\htheta_t$ using the selected subset $X_{sub}$ by gradient descent;
\Until{going through all bags in $\mathcal{D}_{tr}.$}
\State Get validation loss by $L(\htheta_t) \gets \frac1m \sum_{j=1}^m \ell_j^v(\htheta_t)$ on $\mathcal{D}_{va}$;
\State Obtain $s_t \gets H_t^{-1}\nabla_{\theta}L(\htheta_t)$ by stochastic estimation as done in Eq. \eqref{eq:stochasticest};
\EndFor
\end{algorithmic}
\end{algorithm}

\section{Dataset Statistics} \label{app:2}

\begin{table}[htbp]
  \centering
  \caption{Data statistics of used two NYT datasets. ``\# Pos", ``\# Ins", ``\# Rel": number of postive bags, instances and relations, respectively.}
\setlength{\tabcolsep}{1.5mm}{
    \begin{tabular}{l|r|r|r|r}
    \hline
          & \multicolumn{2}{c|}{{ NYT-SMALL}} & \multicolumn{2}{c}{{ NYT-LARGE}} \\
         & \multicolumn{1}{r}{Train} & \multicolumn{1}{r|}{Test} & \multicolumn{1}{r}{Train} & \multicolumn{1}{r}{Test} \\
\hline
    \# Bags &          65,726  &          93,574  &        281,270  &          96,678  \\
    \# Pos &            4,266  &            1,732  &          18,252  &            1,950  \\
    \# Ins &        112,941  &        152,416  &        522,611  &        172,448  \\
    \# Rel &                26  &                26  &                53  &                53  \\
    \hline
    \end{tabular}}
  \label{tab:nytstat}%
\vskip -0.2in
\end{table}%

\section{General Setups for Training PCNN} \label{app:3}
Following the configurations of previous works, we employ word2vec\footnote{https://code.google.com/p/word2vec/} to extract the word embeddings, to process the raw data. Parameters of PCNN are set according to \cite{zeng2015distant}: window size $d^w=3$, sentence embedding size $d^s=230$, word dimension $d^a=50$ and position dimension $d^p=5$ for fair comparison. During training, we fix the batch size $B=128$, dropout ratio $p=0.5$, and use the ADADELTA \cite{zeiler2012adadelta} with parameters $\rho=0.95$ and $\varepsilon=10^{-6}$ for optimization. Since we find the default hyperparameters already lead superior performance of REIF, we did not make further tuning.

\section{Establishing the Validation Set} \label{app:4}
Due to lacking clean validation set, we utilize automatic selection similar to ARNOR \cite{jia2019arnor}. It takes top 10\% high-frequency patterns of each relation as initial pattern, then takes max 5 new patterns in one loop for each relation in bootstrap procedure. We stop bootstrap until 10\% training samples are involved. Our experiments demonstrate REIF can gain significantly from this automatically built validation set, although it is collected by heuristics and not absolutely clean.

\end{document}